\documentclass{article}
\usepackage[T1]{fontenc}
\usepackage{graphicx}
\usepackage{booktabs}
\usepackage[misc]{ifsym}

\usepackage{mwe}

\usepackage{xcolor}
\usepackage{amsmath} 
\usepackage{amssymb}   %
\usepackage{algorithm}  
\usepackage{algorithmic}  
\usepackage{verbatim}
\usepackage{booktabs}
\usepackage{longtable}
\usepackage{subfigure}
\usepackage{enumitem}
\usepackage{breakcites}
\usepackage{stmaryrd}

\usepackage[hidelinks]{hyperref}
\usepackage[capitalise,nameinlink]{cleveref}
\usepackage{amsthm}

\newcommand{\UnivLit}[0]{\mathbb{L}}
\newcommand{\REMOVED}[2]{L_{{#1}}^{\bar{{#2}}}}

\newcommand{\LIT}[0]{{\cal{L}}}
\newcommand{\F}[0]{{\cal{F}}}
\newcommand{\FSpace}[0]{{\mathbb{F}}}

\newcommand{\Outcome}[0]{{\cal{V}}}

\newcommand{\dom}[1]{{{D_{#1}}}}

\newcommand{\nbfeatures}[0]{m}

\newcommand{\BACK}[0]{{\cal{B}}}
\newcommand{\DS}[0]{{\cal{M}}}

\newcommand{\rules}[0]{{\Delta}}

\newcommand{\conflict}[2]{{#1} \ominus {#2}}
\newcommand{\union}[1]{G_{#1}}
\newcommand{\cover}[3]{{Cover({#1},{#3},{#2})}}

\newcommand{\SA}[0]{\Omega_i}
\newcommand{\SB}[0]{\Theta_i}
\newcommand{\SC}[0]{\Xi_i}
\newcommand{\SD}[0]{\Upsilon_i}

\definecolor{midgrey}{rgb}{0.5,0.5,0.5}
\definecolor{middarkgrey}{rgb}{0.35,0.35,0.35}
\definecolor{darkgrey}{rgb}{0.3,0.3,0.3}
\definecolor{darkred}{rgb}{0.7,0.1,0.1}
\definecolor{midblue}{rgb}{0.2,0.2,0.7}
\definecolor{middarkblue}{rgb}{0.15,0.15,0.575}
\definecolor{darkblue}{rgb}{0.1,0.1,0.5}

\newcommand{\jnoteF}[1]{}
\newcommand{\mnoteF}[1]{}
\newcommand{\jpnoteF}[1]{}

\makeatletter
\newcommand\nparagraph{%
  \@startsection{paragraph}
    {4}
    {\z@}
    {1.5ex \@plus0.5ex \@minus.2ex}
    {-1em}
    {\normalfont\normalsize\bfseries}%
}
\makeatother

\renewcommand{\paragraph}{\nparagraph}

\newcommand{\stdmathcal}[1]{\text{\usefont{OMS}{cmsy}{m}{n}#1}}
\newcommand{\fml}[1]{{\stdmathcal{#1}}}

\newcommand{\mbb}[1]{\ensuremath\mathbb{#1}}

\newcommand{\limply}{\ensuremath\rightarrow}

\newtheorem{definition}{Definition}
\newtheorem{lemma}{Lemma}
\newtheorem{theorem}{Theorem}
\newtheorem{proposition}{Proposition}
\newtheorem{example}{Example}
\newtheorem{corollary}{Corollary}

\newcommand{\institute}[1]{\newcommand{\@institute}{#1}}
\providecommand{\keywords}[1]{\textbf{Keywords:} #1}

\usepackage{authblk} 

\begin{document}

\title{On Trustworthy Rule-Based Models and Explanations}



\author[1]{Mohamed Siala}
\author[2]{Jordi Planes}
\author[3]{Joao Marques-Silva}

\affil[1]{LAAS-CNRS, Universit\'e de Toulouse, CNRS, INSA Toulouse, France}
\affil[2]{Universitat de Lleida}
\affil[3]{ICREA \& University of Lleida}


\maketitle              

\begin{abstract}
  A task of interest in machine learning (ML) is that of ascribing
  explanations to the  predictions made by ML models.
  Furthermore, in domains deemed high risk, the rigor of
  explanations is paramount. Indeed, incorrect explanations can and
  will mislead human decision makers.
  As a result, and even if interpretability is acknowledged as an
  elusive concept, so-called interpretable models are employed
  ubiquitously in high-risk uses of ML and data mining (DM).
  This is the case for rule-based ML models, which encompass decision
  trees, diagrams, sets and lists.
  This paper relates explanations with well-known undesired
  facets of rule-based ML models, which include negative overlap
  and several forms of redundancy. 
  The paper develops algorithms for the analysis of these undesired
  facets of rule-based systems, and concludes that well-known and
  widely used tools for learning rule-based ML models will induce rule
  sets that exhibit one or more negative facets.
  
\end{abstract}

\section{Introduction}~\label{sec:intro}

Explainable Artificial Intelligence (XAI) is a mainstay of trustworthy
AI~\cite{berrada-ieee-access18,pedreschi-acmcs19,cardoso-electronics19,ranjan-acmcs23,finzel-dmkd24}. 
Furthermore, in domains that are deemed of high risk, explanations
should be trustable~\cite{rudin-naturemi19,rudin-ss22,11-rules-Huysmans,14-Freitas}.
%
The importance of explanations and the need to trust those
explanations motivated work on so-called interpretable
models~\cite{rudin-naturemi19,rudin-ss22}, even though it is generally
accepted that a rigorous definition of interpretability is
elusive at best~\cite{lipton-cacm18}.
Rule-based models, which encompass decision
trees~\cite{breiman-bk84}, {diagrams~\cite{22-bdd,23-dd}, 
sets~\cite{clarkn89,16-Lakkaraju-kdd,21-ds-alexey} and
lists~\cite{rivest-ml87,22-gmm-jair}, epitomize interpretable models.}

Work on the induction of rule-based models can be traced at least to
the 1970s~\cite{shwayder-cacm71,rivest-ipl76}, in the concrete case of
decision trees.
\footnote{
Although extremely popular in ML and DM, decision trees found earlier
uses in other domains,
e.g.\ \url{https://en.wikipedia.org/wiki/Phylogenetic_tree} and
\url{https://en.wikipedia.org/wiki/Decision_tree}.
}
Decision trees are widely used in practice and often exemplify
interpretable models~\cite{rudin-naturemi19,rudin-ss22}. The perceived
importance of interpretability has recently motivated the development
of algorithms for learning optimal decision
trees{~\cite{demirovic-jmlr22,20-dt-hu}}.
Decision sets (or rule sets) find a wide range of uses in different
domains~\cite{fuernkranz-bk12,fuernkranz-ruleml15,fuernkranz-ecml20,fuernkranz-ruleml20,fuernkranz-adac24,fuernkranz-dmkd24}. As
with decision trees, there has been recent interest in learning
optimal decision sets~\cite{16-Lakkaraju-kdd}.
Decision lists also find many practical uses, but claims about their 
interpretability are harder to justify~\cite{MarquesSilvaI23}.
As a result, this paper studies decision sets, but also decision trees
when viewed as a special case of decision sets.

At present, some of the best-known ML toolkits implement one or more
methods of induction of rule-based
models~\cite{pedregosa-jmlr11,fuernkranz-ecml20,zupan-jmlr13}.
Nevertheless, it has been argued~\cite{MarquesSilvaI23} that rule-based
methods, although easier to fathom by human-decision makers, still
require explanations to be computed. (Otherwise, human decision-makers 
would be expected to manually solve NP-hard function
problems~\cite{MarquesSilvaI23}.)
Therefore, a key question is: \emph{for rule-based models, when can
explanations be computed trivially, such that a human decision-maker 
can manually produce an explanation?
}

This paper shows that rigorous explanations can be found manually
whenever some undesired facets of decision sets are nonexistent.
Concretely, the paper relates easy-to-compute explanations with the
non-existence of \emph{negative overlap}, i.e.\ the existence of cases
where two or more rules can fire that predict different values.
Furthermore, the non-existence of redundant literals in rules is shown
to be a necessary condition for minimality of explanations.

Given this state of affairs, the paper then investigates whether
existing ML toolkits are able to learn rule-based models that avoid
the aforementioned negative facets. As the results demonstrate, this
is not the case.
In addition, the paper investigates whether model-agnostic methods
targeting feature selection (i.e.\ that produce rules as explanations)
are capable of preventing negative overlap (i.e.\ the most worrisome
negative facet).
Unfortunately, as the results show, this is also not the case with the
well-known explainer Anchor~\cite{18-aaai-Ribeiro}.


\jnoteF{This is hard to write because the message is not yet explicit.}


\paragraph{Contributions.}
%
The paper studies decision sets,%
\footnote{Decision trees are a special case of a decision set, and so
we also present experiments on decision trees. However, we opt not
to address decision lists due to the intrinsic difficulties with their
explanation~\cite{MarquesSilvaI23}.}
concretely the problem of \emph{negative overlap}, i.e.\ when two
rules that predict different classes fire, but also the existence of
local or global redundancies of literals in rules.
The paper develops algorithms for deciding the existence of negative
overlap, but also for deciding local and global redundancy.
Furthermore, the results in the paper take into account possible
constraints on the inputs.
The paper then relates these negative facets of decision sets with the
ability of human decision-makers to manually produce rigorous
explanations, namely abductive explanations. In addition, the
experiments confirm that implemented rule-learning algorithms in
well-known toolkits exhibit the negative facets of decision sets,
thus complicating (complexity-wise) the computation of rigorous
explanations.

\paragraph{Organization.}
%
The paper is organized as follows. \cref{sec:prelim} introduces the
notation and definitions used throughout the paper.
\cref{sec:relw} briefly comments on related work.
\cref{sec:rdd} details the paper's main contributions.
\cref{sec:res} reports on the experimental results.
Finally, \cref{sec:conc} concludes the paper.


\section{Background}~\label{sec:prelim}

The notation and definitions used throughout the paper are adapted
from past
works~\cite{16-Lakkaraju-kdd,2021-sat-handbook-biere,22-jair-izza}. 

\paragraph{Propositional Logic and
  Generalizations~\cite{2021-sat-handbook-biere}.}

{
    Let $X= \{x_1, \ldots, x_n \}$ be a set of Boolean variables.
A literal is a Boolean variable or its negation. 
A clause $C$ is a disjunction of literals and a cube $L$ is a conjunction of literals. 
We use the notation $l_i \in C$ (respectively $l_i \in L$) 
if $C= l_1 \vee \ldots \vee l_k$ (respectively $L= l_1 \wedge \ldots \wedge l_k$).
A conjunctive normal form (CNF) formula $F$ is a conjunction of clauses. 
That is, $F= C_1 \wedge \ldots \wedge C_k$ where $C_j$ is a clause. 
In this case, we use the notation $C_j \in F$. 
Note by definition that a clause/cube is a CNF. 
An assignment $v= (v_1, \ldots v_n )$ is a point in $\{0,1\}^n$. 
If $F=C_1 \wedge \ldots \wedge C_k$ is a CNF, $v \models F$ iff 
$\forall C_j \in F, \exists x_i \in C_j$ such that $v_i=1$ 
or $\exists~\neg x_i \in C_j$ such that $v_i=0$. 
If $\exists~v \in \{0,1\}^n$ such that $v \models F$ 
then $F$ is said satisfiable, otherwise unsatisfiable. 
If $F_1$ and $F_2$ are two CNF formulas,
 $F_1 \models F_2$ iff
 $v \models F_1 \implies v \models F_2$. 
Note that $F_1 \models F_2$ iff $F_1 \wedge \neg F_2$ is unsatisfiable. 
Given a CNF formula $F$, the satisfiability problem (SAT) 
asks if $F$ is satisfiable.
SAT solvers are highly deployed in practice to answer SAT related queries, 
such as finding satisfying assignments or proving
unsatisfiability~\cite{2021-sat-handbook-biere}.
Furthermore, extensions of propositional to more expressive logics can
be handled by considering Satisfiability Modulo Theories
(SMT)~\cite{2021-sat-handbook-biere}.
}

\jnoteF{Should we change the title of the paragraph to something less
  specific? Perhaps ``\emph{Automated reasoning}'' or
  ``\emph{Constraint programming}'' or ...?
}

\mnoteF{It was possible to write the theoretical part in a more general setting but 
 to avoid length-related issues
and for simplicity (especially for this conference),
I wrote everything with SAT and therefore the background in SAT to more concise. 
But of course feel free to edit! 
}

\jnoteF{I am not convinced. The examples used do not exhibit boolean
  features. Framing the discussion at the boolean level could raise
  objections from the reviewers.}

  \mnoteF{The literals that are used in the decision sets are Boolean. 
  However, the features of the dataset are not necessary Boolean. 
  For instance, in the decision set we can have the literal "x > 0.8837667" 
  where $x$ is a numerical feature. 
  }
  
  \mnoteF{I can change the "Boolean Satisfiability" 
  paragraph to "Propositional Logic." 
  and talk about propositional formulas 
  in general instead the cnf restriction. 
  }

\paragraph{Machine Learning.}

We consider rule-based models for classification and regression 
{that can be represented as 
a set of unordered rules.}
Let $\F= \{1,\ldots \nbfeatures\}$ be a set of features where 
each feature $i$ takes values from a domain $\dom{i}$. 
The feature space is the Cartesian product of the domains 
$\FSpace = \dom{i} \times \ldots \times \dom{m}$. 
The outcome space (i.e., classes for classification and numerical values for regression)
is denoted by $\Outcome$. 
{A dataset is a set  
$\{ (x,o) ~|~ x \in \FSpace \wedge o \in \Outcome \}$, and where
  $x=(x_1,\dots,x_m)$.
}
A literal represents a condition on the values of a feature.
We use $\UnivLit$ to represent the universe of literals. 
{A background knowledge  
$\BACK$ is a propositional formula over literals
from $\UnivLit$ that specifies
the conditions that any arbitrary point in feature
 space must comply with.
In other words, a point in feature space
 $x$ is \emph{valid} 
iff $x \models \BACK$.}
{We assume in the rest of the paper that $\BACK$ is given as a CNF.}
For example, consider a dataset representing
 individuals and the two literals
$l_1 := employed$, $l_2 := salary > 50k$. 
The background knowledge $\BACK$ can contain the clause 
$l_1 \vee \neg l_2$ to model the fact that an unemployed individual 
cannot have a salary greater than $50k$. 
Note that $\BACK$ can be a tautology
 (for instance when no condition is given).
 In this case, any arbitrary point in feature space is a valid. 
A user can also miss certain constraints she is not aware of. 
{
    Let $\lambda \notin \Outcome$ be a dummy value. 
A supervised ML (classification or regression) model $\kappa$ 
is a mapping from $\FSpace$ to $\{\lambda\} \cup \Outcome$ such that 
$\kappa(x) = \lambda$ iff $ x \not\models \BACK$.
}

{A rule $R_i$ is a pair $(L_i,o_i)$ such that
$L_i$ 
is a conjunction of literals (i.e., cube)
 from $\LIT \subseteq \mathbb{L}$ and $o_i \in \Outcome$. 
 $R_i$ fires on $x \in \FSpace$ iff $x\models L_i$. 
With a slight abuse of notation we shall sometimes 
use $L_i$ as the subset of $\LIT$ formed by the literals in $L_i$. 
A decision set $\DS$ is a set of rules $\DS= \{R_1, \ldots , R_r  , R_{r+1} \}$
such that $\forall i \leq r, L_i \neq \emptyset$ 
and $L_{r+1}= \emptyset$. 
$R_{r+1}$ is called the default rule. 
We denote $\rules(o)$ the set $\{R_i | o_i = o \} $. 
$\DS$ is used as an ML model $\kappa_{\DS}$ as follows:
$$
\kappa_{\DS}(x)=
\begin{cases} 
    \lambda \notin \Outcome & \text{if } x \not\models \BACK \\
    o_{r+1} & \text{if no rule fires on $x$}\\
    o & \text{if } \{o\} = \{o_i ~|~ R_i \text{ fires on } x \} \\ 
    \text{Tie-breaking strategy otherwise} 
\end{cases} 
$$

{
Note that decision trees (DTs), decision diagrams (DDs), random forests (RFs) 
and boosted trees (BTs), can be seen as decision sets
where each path represents a rule. 
Clearly, in such models the default rule never fires. 
In the case of DTs and DDs, each input 
fires exactly one rule (since it follows exactly one path). 
Thus, no tie-breaking strategy is needed. 
This is not the case with RFs and BTs since each input fires one rule on each tree. 
Therefore, a tie-breaking strategy is needed. 
}

\mnoteF{The following comment is addressed. Let me know if it is clear with the new definition.}

\jnoteF{%
  Regarding the definition of $\kappa$ above:\\
  When prediction is $o_{i+1}$, what is $o_{i+1}$, i.e.\ what is $i$
  in that case?\\
  When prediction is $o_i$, why must $R_i$ be the \emph{only} rule? If
  two or more rules predict $o_i$, then prediction should be $o_i$.\\
  We only need tie break two or more rules fire and their predictions
  differ.\\
  This should be clarified...
}

We extend the notion of cover and overlap 
from~\cite{16-Lakkaraju-kdd} 
by considering the background knowledge $\BACK$ and the input space. 

\begin{definition}[Cover]
    Given $X\subseteq \mathbb{F}$ and background knowledge $\BACK$, 
    $\cover{X}{L}{B} = \{x ~|~ x \in X \wedge x\models \BACK \wedge x \models L\}$.
\end{definition}

\begin{definition}[Overlap]
    Given a background knowledge $\BACK$, 
    two rules $R_i$ and $R_j$ such that $i, j \leq r$ overlap in $ X \subseteq \mathbb{F}$ 
    iff $\cover{X}{L_i}{B} \cap \cover{X}{L_j}{B} \neq \emptyset $.
\end{definition}

We say that $R_i$ and $R_j$ positively (respectively negatively) overlap 
if they overlap and $o_i = o_j$ (respectively $o_i \neq o_j$). 
We use the notation $R_i \ominus R_j$ if $R_i$ and $R_j$ negatively overlap. 
{
  Observe that DTs and DDs exhibit no overlap since each 
input is captured by exactly one rule. 
This is not the case for RFs and BTs, since 
each input fires exactly one rule from each tree. 
Thus, overlaps may occur only between rules from different trees. 
}

\paragraph{Formal Explanations~\cite{22-jair-izza,darwiche-lics23}.}
%

Most approaches to explainability target at instance, i.e.\ a pair
$(x,c)$ with $x\in\mbb{F}$ and $c\in\fml{V}$. 
We use $\kappa$ throughout the paper to denote a machine learning model. 
Given an instance $(v,c)$, with $c=\kappa(v)$, a weak abductive
explanation (WAXp) is a subset $\fml{X}$ of the features $\fml{F}$
which, if assigned the values dictated by $v$, is sufficient for the
classifier to output prediction
$c=\kappa(v)$~\cite{22-jair-izza,darwiche-lics23}:
\begin{equation} \label{eq:waxp}
  \forall(x\in\mbb{F}).\left[\bigwedge\nolimits_{i\in\fml{X}}(x_i=v_i)\limply(\kappa(x)=c)\right]
\end{equation}
A subset-minimal WAXp is an \emph{abductive explanation} (AXp).
Recent work demonstrated the need for explaining interpretable models,
including decision trees~\cite{22-jair-izza} and
lists~\cite{MarquesSilvaI23}.
To the best of our knowledge, past work did not investigate
formal explanations for decision sets.

Furthermore, the definition of WAxp (see~\eqref{eq:waxp}) can be
generalized to account for literals involving other relational
operators~\cite{22-jair-izza} (e.g.\ relational operators taken from
$\{\in,\ge,>,<,\le\}$). In addition, constraints on the
inputs~\cite{rubin-aaai22,24-ijcai-Audemard} can be accounted for by conjoining a set of
constraints $\fml{C}_{\fml{B}}$. For example, these constraints allow
capturing the background knowledge introduced earlier in this section.
Concretely, we write that $\fml{C}_{\fml{B}}(x)$ holds true iff $x$
respects the background knowledge, i.e.\ $x\models \BACK$.

\jnoteF{We should agree how to represent points of $\mbb{F}$. If we do
  not use boldface, then $x$ and $x_i$ are easy to be confused with
  each other.}

\section{Related Work}\label{sec:relw}

The learning of rule-based models has been the subject of research
since the 1970s~\cite{shwayder-cacm71,rivest-ipl76}.
The importance of the topic, especially given their widely accepted
interpretability, has motivated recent work on learning decision
sets~\cite{fuernkranz-ecml20,fuernkranz-adac24,fuernkranz-dmkd24} and
(optimal) trees~\cite{demirovic-jmlr22}.
These earlier works were motivated by the accepted belief that
decision trees, sets and lists are
interpretable~\cite{breiman-ss01,rudin-naturemi19,rudin-ss22}.
Accounts of methods for learning decision sets and lists
include~\cite{fuernkranz-bk12,fuernkranz-ruleml15}.

Motivated by the elusive nature of
interpretability's definition~\cite{lipton-cacm18}, recent
work~\cite{MarquesSilvaI23} uncovered practical difficulties in
computing and/or using so-called interpretable models as
explanations. For example, it has been shown that paths in decision
trees can be arbitrarily redundant (on the number of features) when
compared with an AXp~\cite{22-jair-izza}. Similarly, the computation
of an AXp for a decision list equates with solving an NP-hard
problem~\cite{MarquesSilvaI23}, i.e.\ something that is in general
beyond the capabilities of a human decision-maker.
Nevertheless, past work did not address formal explanations for
decision sets, arguably because of the existence of 
negative overlap.

Although the paper assesses rule-based methods using formal
explanations, XAI is better-known by the use of 
model-agnostic
methods~\cite{berrada-ieee-access18,pedreschi-acmcs19,cardoso-electronics19,nguyen-air22,ranjan-acmcs23,finzel-dmkd24}. Well-known
examples include LIME~\cite{guestrin-kdd16},
SHAP~\cite{lundberg-nips17} and
Anchors~\cite{18-aaai-Ribeiro}.
Since so-called interpretable models have been proposed for high-risk
uses of ML, we focus on rigorous (i.e.\ formal) explanations.

The main results of this paper, namely the direct relationship between
easy-to-compute explanations and the non-existence of well-known
negative facets of rule-based models, are novel. The observation that
rule-based models, obtained with well-known toolkits, exhibit those
negative facets, is also a novel result, to the best of our knowledge.

%

\section{Overlap and Redundancy}~\label{sec:rdd}
In this section, we let $\BACK$ be a background knowledge 
and $\DS= \{R_1, \ldots, R_r, R_{r+1}\}$
be a decision set where $R_{r+1}$ is the default rule such that 
each rule $R_{i\leq r}$ fires on at least one valid input (w.r.t. $\BACK$). 
As mentioned in the introduction, we provide 
a formal framework to address the following questions: 
(i)~How can we generate all (negative) overlap?; 
(ii)~Is rule $R_i$ redundant in $\DS$?; 
and (iii)~Is literal $l$ redundant in a given rule?. 

We use Example~\ref{example:ds} throughout the paper to illustrate 
the different concepts.
\begin{example}\label{example:ds}
$\BACK$ is background knowledge that encodes the following constraints (in a CNF):
    $(salary > 0) \leftrightarrow (age \geq 18) $; 
     $ (size = 140) \rightarrow (size > 120) $; 
     $(weight > 90) \rightarrow (weight \geq 85) $; and 
         $(weight \geq 85) \rightarrow (weight > 80) $.
The decision set contains the following rules:
\begin{itemize}
    \item $R_1 = ((salary > 0) \wedge (size \neq 140) \wedge (age > 10)
     \wedge (color = blue) \wedge (weight > 80),\ 1)$
    
     \item $R_2 = ((salary > 0) \wedge (size = 140),\ 1)$
    \item  $R_3 = ((salary > 0) \wedge (weight > 90),\ 1)$
    \item $R_4 = ((size > 120) \wedge (weight < 85),\ 0)$
\end{itemize}
\end{example}

\subsection{Overlap}\label{subsec:overlaps}
We start by giving a sufficient and 
necessary condition to check if two rules negatively overlap. 
\begin{lemma}[Overlap Check]\label{lemma:ds_negative_overlap}
    Two rules $R_i$ and $R_j$ overlap 
     iff  
    $ \BACK \wedge L_i \wedge L_j$ is satisfiable.
\end{lemma}
\begin{proof}
    $ \BACK \wedge L_i \wedge L_j$ is satisfiable iff 
    $\exists x\in \FSpace, x\models  \BACK \wedge  L_i$ 
    and $x\models \BACK \wedge L_j$. 
    This is equivalent to $\exists x\in \FSpace, \{x \} \in \cover{\FSpace}{L_i}{B} \cap \cover{\FSpace}{L_j}{B}$.
    The latter means that $R_i$ and $R_j$ overlap. \qed
 \end{proof}

In Example~\ref{example:ds}, 
one can use Lemma~\ref{lemma:ds_negative_overlap}
 to show that $R_3$ and $R_4$ do not overlap,
  in contrast to $R_1$ and $R_4$, which do.

We consider now the question of generating all negative overlap. 
 Algorithm~\ref{algo:overlap} finds all pairs of rules that exhibit 
 a negative overlap. 
 We use $GetList(o_1,o_2, \ldots o_r)$ 
as a function that computes a list that contains 
the distinct values in $\{o_1, \ldots o_r\} $.  

\begin{algorithm}[t]
   \caption{Negative Overlap Pairs\label{algo:overlap}}
   \begin{algorithmic}[1]
       \STATE \textbf{Function: } $Pairs$ 
       \STATE \textbf{Input:} $\FSpace, O,  \DS=\{R_1, \ldots R_r \}, \BACK$ 
       \STATE \textbf{Output:} 
        $  \Pi = \{ (i,j) ~|~ \conflict{R_i}{R_j} \} $  \\

       \STATE $\Pi = \emptyset$
       \STATE $\Psi=GetList(o_1,o_2, \ldots o_r)$\\
        \STATE $g = | \Psi |$ 
       \FOR{$a$ in $1, \ldots g-1 $} 
            \FOR{$b$ in $a+1, \ldots g$} 
                \FOR{$R_i$ in $\rules({\Psi(a)})$} 
                    \FOR{$R_j$ in $\rules({\Psi(b)})$} 
                       \IF{$ \BACK \wedge L_i \wedge L_j$ is SATISFIABLE}
                            \STATE $\Pi \gets \Pi \cup  \{ (i,j)\} $ \label{line:check}
                        \ENDIF
                    \ENDFOR
                \ENDFOR
            \ENDFOR
        \ENDFOR
       \STATE \textbf{Return} $\Pi$
   \end{algorithmic}
   \end{algorithm}

   Algorithm\ref{algo:overlap} terminates because each pair of rules will be visited at most once. 
   The correctness  of Algorithm\ref{algo:overlap} 
   follows from the fact that each pair $(R_i,R_j)$ such that $o_i \neq o_j$
   is visited exactly once in Line~\ref{line:check}. 
   The complexity of Algorithm~\ref{algo:overlap} is $O( |\DS|^2 \times f(\DS))$
   where  $f(M)$ is the worst complexity of 
   $\BACK \wedge L_i \wedge L_j$ for an arbitrary pair of rules $(R_i, R_j)$. 
   This observation follows from the fact that 
   computing $GetList$ can be naturally be done in 
   $O(|\DS|) $ and the fact that 
   the satisfiability check in Line~\ref{line:check} is called at most once
   for each pair $(R_i,R_j)$. 

Finally, one might ask whether the default rule can be triggered. 
Proposition~\ref{prop:total} shows that this can be achieved with one SAT call.

\begin{proposition}[Default Rule Application]\label{prop:total}
The default rule is triggered iff 
$\BACK \wedge \neg L_1 \ldots \wedge \neg L_r$  is satisfiable. 
\end{proposition}
\begin{proof}[Sketch]
    No rule fires on a solution to $\BACK \wedge \neg L_1 \ldots \wedge \neg L_r$. 
    \qed
\end{proof}

\subsection{Redundancy}
In order to study rule and literal redundancy, 
we provide a formal definition of decision sets equivalence. 
We denote by $S_\DS(o) = \cup_{R_i \in \rules(o)} \{ x\in \FSpace\ |\ x \models \BACK \wedge L_{i}\}$.

\begin{definition}[Decision Set Equivalence]~\label{def:ds_equiv}
    Let $\DS_1$ and $\DS_2$ be two decision sets defined over the same feature space $\FSpace$ 
    and output $\Outcome$ and having the same default rule. 
        $\DS_1$ is equivalent to $\DS_2$ iff 
        $\forall o \in \Outcome, S_{\DS_1}(o)=S_{\DS_2}(o)$.
        \end{definition}

The following lemma is an immediate consequence of Definition~\ref{def:ds_equiv}. 
\begin{lemma}[Lemma Decision Set Equivalence]~\label{lemma:ds_equiv}
      Let $\DS_1$ and $\DS_2$ be two equivalent decision sets that
       exhibit no negative overlap and let $\BACK$ be a 
       background knowledge. 
       Then $\forall x \models \BACK, \kappa_{\DS_1}(x) = \kappa_{\DS_2}(x)$.
\end{lemma}

We introduce the notion of rule redundancy 
to capture the fact that removing a given rule from a decision set
 leads to an equivalent decision set.

\begin{definition}[Rule Redundancy]\label{def:rule_rdd}
    A rule $R_i$ is redundant in $\DS$ iff 
    $\DS \setminus R_i$ is equivalent to $\DS$ 
    \end{definition}
    
    Let $\union{i} = \rules(o_i)\setminus\{R_i\} = 
    \{R_{i_1}, \ldots, R_{i_z}\} $ where $R_{i_m} = (L_{i_m}, o_{i_m})$. 
 
\begin{proposition}[Rule Redundancy Check]\label{prop:rule_rdd}
    A rule $R_i$ is redundant in $\DS$ iff 
    $\BACK \wedge L_i \models L_{i_1} \vee \ldots \vee L_{i_z}$. 
\end{proposition}
\begin{proof}
    Let $\DS*=  \DS \setminus R_i$. 
    Clearly $R_i$ is redundant in $\DS$ iff 
    $S_{\DS}(o_i)=S_{\DS^*}(o_i)$. 
    In other words, iff 
    $
    \cup_{R_j \in \rules(o_i)} \{ x\in \FSpace\ |\ x \models \BACK \wedge L_{j}\} = 
    \cup_{R_j \in \rules(o_i)\setminus{R_i} } \{ x\in \FSpace\ |\ x \models \BACK \wedge L_{j}\}
    $. 
    The latter is true iff 
    $\BACK \wedge L_i \models L_{i_1} \vee \ldots \vee L_{i_z}$. 
    \qed
\end{proof}

Following Proposition~\ref{prop:rule_rdd}, 
one can check if a rule is redundant with one SAT oracle 
since 
$\BACK \wedge L_i \models L_{i_1} \vee \ldots \vee L_{i_z}$ iff 
$\BACK \wedge L_i \wedge \neg L_{i_1} \wedge \ldots \wedge \neg L_{i_z}$ is unsatisfiable. 
For instance, in Example~\ref{example:ds},
this allows to show that $R_3$ is redundant. 

One can also build an equivalent decision set with no 
redundant rules by checking and removing redundant rules iteratively. 
Note that the order in which the redundant rules are removed matters
as it might return a different decision set at each execution.

We assume in the rest of this section that no rule is redundant. 
Suppose that $L_i$ contains at least two literals and that $l\in L_i$. 
We denote by $\DS^{i}_l = \DS \cup (L_i \setminus {l} , o_i) \setminus R_i $
the decision set identical to $\DS$ except that
 $l$ is removed from $L_i$. 
We give a formal definition of literal redundancy.

\begin{definition}[Literal Redundancy]\label{def:lit_rdd}
     A literal $l$ is redundant in $L_i$ iff
      $ l \in L_i$ and 
      $\DS^{i}_l $  is equivalent to $\DS$.
    \end{definition}

Informally speaking, a literal is redundant in $L_i$ iff its 
removal from $L_i$ leads to an equivalent decision set. 
In the following we prove that there are only two cases of redundancies 
that we call local and global redundancies, and 
we show sufficient and necessary conditions to find (and remove) them. 
When using $L_i$, we suppose that it contains at least two literals.

We denote by $\REMOVED{i}{l} = L_i \cup \{\neg l \} \setminus \{l\}$. 
We define the following sets to address literal redundancy: 
$\SA= \cup_{R_j \in \union{i}} \{x \in \FSpace\ |\ x \models \BACK \wedge L_j\}$, 
$\SB= \{x \in \FSpace\ |\ x \models \BACK \wedge L_i \setminus \{l\} \}$, 
$\SC= \{x \in \FSpace\ |\ x \models \BACK \wedge L_i\}$, 
$\SD= \{x \in \FSpace\ |\ x \models \BACK \wedge \REMOVED{i}{l}\}$.
By construction, we have: 
\begin{itemize}
    \item $\SB= \SC \cup \SD$ 
    \item $S_\DS(o_i) = \SA \cup \SC$
    \item $S_{\DS^{i}_l}(o_i) = \SA \cup \SB = \SA \cup \SC \cup \SD$
\end{itemize}

\begin{proposition}[Literal Redundancy (1)]
    \label{prop:rdd}
    A literal $l$ is redundant in $L_i$ iff
    $ l \in L_i$ and 
    $\SA \cup \SC = \SA \cup \SB = \SA  \cup \SC \cup \SD$
\end{proposition}
\begin{proof}
    Observe first that $\DS^{i}_l $ is equivalent to $\DS$ 
    iff  $S_\DS(o_i)= S_{\DS^{i}_l}(o_i)$. 
    Therefore, 
    $l$ is redundant in $L_i$  iff 
    $\SA \cup \SC  = \SA \cup \SB = \SA  \cup \SC \cup \SD$.
    \qed
\end{proof}

\begin{lemma}[Local Redundancy]\label{lemma:rddl}
    If $l \in L_i$ and 
    $ \BACK \wedge L_i\setminus\{ l\}  \models l$ then $l$ is redundant in $L_i$.  
This is called local redundancy. 
\end{lemma}
\begin{proof}
If $ \BACK \wedge L_i\setminus\{ l\}  \models l$ 
then 
$\SC = \SB$ and 
thus 
$S_{\DS^{i}_l}(o_i) = \SA \cup \SB = \SA \cup \SC = S_{\DS}(o_i)$. 
Therefore, by Proposition~\ref{prop:rdd}, $l$ is redundant in $L_i$. \qed
\end{proof}

{In Example~\ref{example:ds}, $(age>10)$ is locally redundant in $R_1$.}

Recall that $\union{i}= \rules(o_i)\setminus\{R_i\} = \{R_{i_1}, \ldots, R_{i_z}\}$
and $\REMOVED{i}{l} = L_i \cup \{\neg l \} \setminus \{l\}$. 

\begin{lemma}[Global Redundancy]
    \label{lemma:rddg}
    If $l$ is not locally redundant in $L_i$ 
    and 
    $ \BACK \wedge \REMOVED{i}{l} \models L_{i_1} \vee \ldots, \vee L_{i_z} $, 
    then $l$ is redundant in $L_i$.  
    This is called global redundancy.
\end{lemma}
\begin{proof}
    If 
     $ \BACK \wedge \REMOVED{i}{l} \models L_{i_1} \vee \ldots, \vee L_{i_z} $
    then 
   $\SD \subseteq \SA$. 
   Thus, since $\SB =\SC \cup \SD$, we have 
    $S_{\DS^{i}_l}(o_i) = \SA \cup \SB = \SA \cup \SC \cup \SD = \SA \cup \SC = S_{\DS}(o_i)$. 
    Therefore, by Proposition~\ref{prop:rdd}, 
    $l$ is redundant in $L_i$.  \qed
    \end{proof}

    {In Example~\ref{example:ds}, $ (size \neq 140 )$ is globally redundant in $R_1$.}

\begin{theorem}[Literal Redundancy (2)]~\label{theorem:rdd}
        A literal $l\in L_i$ is redundant 
        iff it is locally redundant or globally redundant. 
    \end{theorem}
    \begin{proof}
        $\Longrightarrow$: 
        If $l$ is redundant, then by Proposition~\ref{prop:rdd} we have 
        $\SA \cup \SC = \SA  \cup \SC \cup \SD$. 
        Observe that $\SD \cap \SC = \emptyset$. 
        This is because 
        if $x \in \SD \cap \SC$, 
        then $x\models B \wedge L_i \wedge \REMOVED{i}{l}$ 
        which is false because
         $L_i \wedge \REMOVED{i}{l}$ contains $l$ and $\neg l$. 
        Therefore, there are only two cases
         for  $\SA \cup \SC = \SA  \cup \SC \cup \SD$ to hold.
          Either $\SD = \emptyset$ or $\SD \neq \emptyset $ and $\SD \in \SA$. 
          The first case is true iff $\BACK \wedge L_i\setminus\{ l\}  \models l$, that is, 
           $l$ is locally redundant. 
           The second case is true iff 
           $ \BACK \wedge \REMOVED{i}{l} \models L_{i_1} \vee \ldots, \vee L_{i_z} $, 
           that is, $l$ is globally redundant
           $\Longleftarrow$: trivial. 
        \qed
    \end{proof}

    \begin{corollary}[Assessing Literal Redundancy]\label{corollary:full}
        A literal $l \in L_i$ is redundant iff 
        one of the following conditions holds:
        \begin{enumerate}
            \item \textbf{Local redundancy}: 
            \[
                \BACK \wedge (L_i \setminus \{l\}) \wedge \neg l \text{ is unsatisfiable.}
            \]
            \item \textbf{Global redundancy}: (1) does not hold, and 
            \[
                \BACK \wedge  \REMOVED{i}{l} \wedge \neg L_{i_1} \wedge \dots \wedge \neg L_{i_z} \text{ is unsatisfiable.}
            \]
        \end{enumerate}
    \end{corollary}
\begin{proof}
    Immediate from Theorem~\ref{theorem:rdd}  and Lemmas~\ref{lemma:rddl} and~\ref{lemma:rddg}.
    \qed
\end{proof}

Corollary~\ref{corollary:full} can be used to iteratively 
remove redundant literals, thus building decision sets
 with no rules/literal redundancies. 
 It should be noted that different removal orders might lead to 
 different decision sets.

\begin{example}
    Suppose that $\BACK = (b \vee w) \wedge (\neg d \vee f) $ and  $\DS=\{R_1, R_2, R_3\}$
    where  $R_1 : (L_1 = a \land b, o_1)$, $R_2 : (L_2 = a \land w, o_1)$, and 
     $R_3 : (L_3 = c \land d \wedge f, o_2)$. 
     \begin{itemize}
        \item $ \BACK \wedge L_3\setminus\{ f\}  \models f$. Therefore, $f$ is locally redundant in $L_3$.
        \item  $\REMOVED{1}{b} = a \land \neg b$, and 
        $\BACK \wedge \REMOVED{1}{b} =  (b \lor w) \land a \land \neg b
         \equiv a \land \neg b \land w$.
        Thus $\BACK \wedge \REMOVED{1}{b} \models R_2$ and therefore $b$ is globally redundant in $L_1$. 
     \end{itemize}
\end{example}

\subsubsection{Relation to Abductive Explanations.}

\begin{proposition}\label{prop:wax-rdd}
 {Suppose that $L_k \subseteq \{ x_i=v_i^j ~|~ i \in [1,m], v_i^j \in \dom{i} \}$.}
 If $R_k=(L_k,o_k)$ fires on $x$, and there is no negative overlap involving $R_k$,
  then the features used in $L_k$ represent a WAXp.
\end{proposition}
\begin{proof}
By construction.  \qed
\end{proof}

\begin{proposition}\label{prop:ax-rdd}
     {Suppose that $L_k \subseteq \{ x_i=v_i^j ~|~ i \in [1,m], v_i^j \in \dom{i} \}$.}
  If $R_k=(L_k,o_k)$ fires on $v$, there is no negative overlap, and
  $L_k$ contains no (global or local) redundant literal, then the
  features from $L_k$ represent a AXp.
  \qed
\end{proposition}
\begin{proof}
Suppose by contradiction that the features from $L_k$ do not define an AXp. 
Then there is a literal $l \in L_k$ such that 
$\forall x \in \FSpace$ such that $x \models \BACK$,
if $x\models L_k\setminus \{l\}$, 
then $\kappa_{\DS(x)} = o_k$. 
In this case, $\DS^{k}_l $ (i.e., the decision set identical to $\DS$
 except for $L_k$ which is replaced with $L_k\setminus\{l\}$) is equivalent to $\DS$.
 Then, by Definition~\ref{def:lit_rdd}, $l$ is redundant, 
 thus the contradiction. 
  \qed
\end{proof}

Observe that, if the conditions of~\cref{prop:ax-rdd} hold, then
the literals in $L_k$ represent an AXp, and so can be identified
manually by a human decision-maker. Otherwise, as proved in earlier
work for the concrete case of decision lists~\cite{MarquesSilvaI23},
finding an AXp is computationally hard.



\section{Experiments}~\label{sec:res}

We evaluate the different desired properties on different use cases 
including decision sets, decision trees, and anchor explanations. 
All SAT calls are performed using the PySAT toolkit%
\footnote{%
\url{https://pysathq.github.io/}.}
with its default configuration~\cite{itk-sat24,imms-sat18}. 
All experiments run on AppleM1 Pro that has 32G memory and 
8 cores.

\paragraph{Prediction Models \& Datasets.}

In order to make our evaluation as broad and as unbiased as possible, we selected 
datasets from the UCI machine learning 
repository~\footnote{\url{https://archive.ics.uci.edu}} 
with the parameters: ${\cal{P}} = {(Task, Min, Max, Nb, Types)}$ 
on each use case (whenever relevant) where: 
\begin{itemize}[nosep,topsep=1.5pt]
    \item $Task \subseteq \{classification, regression\}$ is the prediction task 
    \item $Min$ (respectively $Max$) is the minimum (respectively $maximum$) size of the dataset. 
    \item $Nb$ is the minimum number of inputs of each class present in the dataset in case of classification.
    \item $Types \subseteq \{numerical,~binaly\}$ is the type of features.
\end{itemize}

We describe the different 
prediction models along with their tailored setting. 
\begin{itemize}[nosep,topsep=1.5pt]
    \item \textbf{Orange (v3)}\footnote{\url{https://orangedatamining.com}}: 
    a library to learn decision sets for 
    classification. 
    The datasets are selected using the parameters 

     ${\cal{P}} = {(\{classification\}, 100, 10^6, 20, \{binary,  numerical\})}$. 

     \item \textbf{Boomer~\cite{fuernkranz-ecml20}\footnote{\url{https://github.com/mrapp-ke/MLRL-Boomer}}}: 
     A library for learning gradient boosted multi-label classification rules. 
     We use the default 
     Boomer datasets\footnote{\url{https://github.com/mrapp-ke/Boomer-Datasets}}.  
    
    \item \textbf{scikit-learn (v1.6.1)}\footnote{\url{https://scikit-learn.org/stable/}}
    and \textbf{Interpretable AI (IAI)}\footnote{\url{https://www.interpretable.ai/}}
 to learn decision trees (DTs) for 
 {classification and regression}. 
scikit-learn learns trees in a greedy way 
with no guarantee of optimality whereas 
IAI learns optimal decision trees. 
The parameters used for the datasets are

${\cal{P}} = {(\{classification, regression\}, 100, 4* 10^6, 20, \{binary,  numerical\})}$. 

\end{itemize}

\paragraph{Background Knowledge.}
In our empirical study, 
the Boolean variables that are used in the different decision sets 
represent a domain relation of the form $(x_f \bowtie v_f)$ 
where $\bowtie\,\in\{=,>,\ge,\le,<\}$ 
for some $f \in \FSpace$ and $v_f \in \dom{f}$. 
We implemented a general purpose procedure to generate
a background knowledge $\BACK$ for each use case
 that maintains domain coherence. 
For instance, 
if $(length > 30)$ and $(length = 17)$ appear in a 
decision set, then $\BACK$ contains the clause
 $\neg(length = 17) \vee \neg(length > 30)$. 


\jnoteF{I believe you should use parenthesis to indicate that the
  complement applies to the literal}
\mnoteF{Done}

Given a set of rules, 
for each feature $f$, we first compute 
 a list, called $Val_f$, that contains all distinct values 
from the domain of $f$ that are used in the decision set 
(or Anchor explanations). 
 $Val_f$ is increasingly ordered if the values are numerical. 
 We also collect
the set of unary relations used for $f$, denoted by $Relations_f$,
which can be any subset of 
$\{=,>,\ge,\le,<\}$.  
The background knowledge $\BACK$ is then constructed using
 Algorithm~\ref{algo:bc} as a CNF.
   Note that we need only the three operators $>, \geq$, and $=$, since 
   $(x > v)$ is equivalent to $(x \geq v-1)$ and 
   $(x < v)$ is equivalent to $(x \leq v-1)$.
   Algorithm~\ref{algo:bc} follows 
   standard procedures for encoding finite domains into
  SAT~\cite{09-lcg}.

\begin{algorithm}[t]
   \caption{{Domain Constraints As a Background Knowledge~\label{algo:bc}}}
   \begin{algorithmic}[1]
       \STATE \textbf{Function: } Build $\BACK$
       \STATE \textbf{Input:} $Relations_1,\ldots, Relations_m, Val_1, \ldots, Val_m  $
       \STATE \textbf{Output:} 
         A background knowlegde $\BACK$ as a CNF \\

       \STATE $\BACK = \emptyset$
     
       \FOR{$f \in \{1,\ldots, m\}$} 
           \IF{ $`='~\in Relations_f$ }
                \STATE $\BACK \gets \BACK \cup  \{ (f=Val_f[i]) \implies \neg (f=Val_f[j]) ~|~ i<j \in [1, |Val_f|]\} $ 
            \ENDIF

            \IF{ $`>'~\in Relations_f$ }
                \STATE $\BACK \gets \BACK \cup 
                 \{ (f>Val_f[i+1]) \implies (f>Val_f[i]) ~|~ i \in [1, |Val_f| -1 ]   \} $ 
            \ENDIF

        \IF{  $ `\geq'~\in Relations_f$} 
           \STATE $\BACK \gets \BACK \cup 
             \{ (f \geq Val_f[i+1]) \implies (f \geq Val_f[i])  ~|~  i \in [1, |Val_f| -1]: \} $
        \ENDIF

        \IF{  $ \{`=' , `\geq' \}~\subseteq Relations_f$} 
           \STATE $\BACK \gets \BACK \cup 
             \{ (x=Val_f[i]) \implies (x \geq Val_f[i]) ~|~   i \in [1, |Val_f|] \}$
             \STATE $\BACK \gets \BACK \cup 
             \{  (x=Val_f[i]) \implies \neg (x \geq Val_f[i+1])  ~|~  i \in [1, |Val_f| -1 ] \}$
        \ENDIF

        \IF{$\{`=' , `>' \}~\subseteq Relations_f$} 
           \STATE $\BACK \gets \BACK \cup 
             \{ (x=Val_f[i]) \implies \neg (x > Val_f[i])  ~|~  i \in [1, |Val_f|] \}$

           \STATE $\BACK \gets \BACK \cup 
             \{  (x=Val_f[i+1]) \implies (x > Val_f[i])   ~|~  i \in [1, |Val_f| -1 ] \}$
        \ENDIF

        \IF{  $\{`\geq' , `>' \}~\subseteq Relations_f$} 
           \STATE $\BACK \gets \BACK \cup 
             \{  (x>Val_f[i]) \implies (x \geq Val_f[i]) ~|~   i \in [1, |Val_f|] \}$
        \ENDIF
        
        \IF{  $\{`=' ,`\geq' , `>'  \}~\subseteq Relations_f$} 
           \STATE $\BACK \gets \BACK \cup 
             \{ (x\geq Val_f[i]) \implies (x =Val_f[i]) \vee (x>Val_f[i])  ~|~  i \in [1, |Val_f|] \}$
        \ENDIF

        \ENDFOR

       \STATE \textbf{Return} $\BACK$
   \end{algorithmic}
   \end{algorithm}


\paragraph{Learning Setting.} 
For Orange, sickit-learn, and IAI, 
a grid search is used to select the best values for the 
maximum rule length, the minimum covered examples per rule, 
among others. 
Each dataset used with Orange, sickit-learn, and IAI 
is split into $80\%$ for training and $20\%$ for testing. 
Boomer's learning parameters are the default ones except 
for the maximum number of rules that we fix 
to $100$
 with one label classification.
 The detailed grid search parameters are given 
 in Table~\ref{tab:grid}.
Cross validation is performed
with $5$ folds for all experiments
using stratified sampling and 
each execution is randomly repeated $4$ times.

\begin{table}[t]
    \centering
    \caption{Grid Search Parameters\label{tab:grid}}
    \resizebox{\textwidth}{!}
    {
    
\begin{tabular}{||c||c|c|c|c|c||} 
\hline
\hline
                & Orange &  Sklearn Class. & Sklearn Reg. & IAI Class. & IAI Reg. \\ \hline\hline
 Beam Width     & 10,30  & -    &    - & -    &    -   \\ \hline
 Min Covered    & 5,15   & -    &    - & -    &     -  \\ \hline
Max Rule Length  & 3,5  & -    &    - &    - &  -   \\ \hline
Criterion       & -  & gini, entropy   &  sqr err, fried mse     &  -  &  mse  \\ \hline
Max Depth       & -   & 3,5,7,9    & 3,5,7,9,11   &    3,5,7 &  3,5,7  \\ \hline
Min Sample Leaf & -   &  5,15,25   & 5,15,25    &    - &  -   \\ \hline
Min Bucket      & -   & -        - & -    &     5,15   &  5,15  \\ \hline  \hline
\end{tabular}

    }
\end{table}

\paragraph{Experimental Pipeline.}

All decision sets that have only one output are discarded. 
For each decision set, we first remove
duplicate rules and rules that never fire. 
After this preprocessing, we run Algorithm~\ref{algo:overlap}
 to find all overlap.
Next, we look for redundant rules then remove them. 
Finally, we compute 
all locally/globally redundant literals. 
We use a timeout of one hour on each decision set 
to find all pairs negative overlap
and rule/literal redundancies.

 \paragraph{Decision Set Statistics.} 
\begin{itemize}[nosep,topsep=1.5pt]
    \item Train: Training accuracy (classification) or Training MSE (regression)
    \item Test: Testing accuracy (classification) or Training MSE (regression)
    \item NR: Number of rules
    \item NP: Cardinality of $\{o_i ~|~ \DS= \cup_i (L_i,o_i)\}$ 
    \item TO: CPU time (s) to find all negative overlap
    \item TB: CPU time (s) to generate the background knowledge
    \item TC: CPU time (s) to find all redundant rules
    \item TR: CPU time (s) to find all redundant literals 
    \item BS: Size of the background knowledge
    \item RS: Sum of the sizes of the rules 
    \item RM: Maximum rule size
    \item NO: Number of negative overlap
    \item PO $=\frac{\text{NO}}{\text{Total}}$: Percentage of
     negative overlap where 
    Total is the total number of pairs of rules associated to different predictions 
\end{itemize}

\paragraph{Model Statistics.}
We report for each prediction model the following: 
\begin{itemize}[nosep,topsep=1.5pt]
    \item DS: The total number of decision sets
    \item EX: The total number of decision sets that timed out
\item IR: Number of instances that admit at least one redundant rule 
 \item IL: Number of instances that admit at least one locally redundant literal    
    \item PL: Percentage of locally redundant literals 
    for instances that admit at least one locally redundant literal 
    \item IG: Number of instances that admit at least one globally redundant literal
    \item PG: Percentage of globally redundant literals 
    for instances that admit at least one globally redundant literal 
\end{itemize}

In the rest of the section, 
we focus on the most important observations. 

\paragraph{Summary.}
Tables~\ref{tab:summary1} and Tables~\ref{tab:summary2} give the full statistics for each 
learning model\footnote{{The detailed results can be found at
 \url{https://siala.github.io/data/2025-ecml/}}}.
Instance-related statistics are averaged for each model.  
Decision sets that are worse than random guess are ignored. 
Instance statistics are averaged for each prediction model. 
Only the results of the experiments that did not reach the timeout are reported. 
The time to generate the background knowledge (TB) is often less than a second. 
The time to find redundant rules (TC) is often few seconds, 
except for some decision sets where it took about a minute. 
 The runtime to find all literal redundancies (TR), however, 
 is much longer. 
 To observe this more accurately, we present in
  Figure~\ref{fig:tr} its box plot across all models. 
  The x-axis is the time in seconds and the y-axis is the TR value for each model. 
  This is expected because every 
  literal is checked for redundancy
   by application of Corollary~\ref{corollary:full}.

\begin{table}[ht]
    \centering
    \caption{Summary of the Results (1)\label{tab:summary1}}
          \begin{tabular}{|c|c|c|c|c|c|c|c|}
\hline
  & DS & EX & NR & NP & TO & TB & TC  \\
\hline
\hline
sklearn classification & 196 & 21 & 35 & 3 & 0 & 0 & 0  \\
\hline
sklearn regression & 28 & 8 & 70 & 69 & 0 & 0 & 8    \\
\hline
IAI classification & 177 & 0 & 17 & 4 & 0 & 0 & 0   \\
\hline
IAI regression & 28 & 0 & 56 & 53 & 0 & 0 & 3   \\
\hline
Orange & 127 & 12 & 175 & 2 & 76 & 0 & 2  \\
\hline
Boomer & 180 & 16 & 97 & 49 & 0 & 0 & 0  \\
\hline
\hline
\end{tabular}

\end{table}

\begin{table}[ht]
    \centering
    \caption{Summary of the Results (2)\label{tab:summary2}} 
          \begin{tabular}{|c|c|c|c|c|c|c|c|c|c|}
\hline
  &  TR & BS & RS & RM & IR & IL & PL & IG & PG  \\
\hline
\hline
sklearn classification  & 94 & 23 & 233 & 5 & 0 & 123 & 7 & 175 & 33  \\
\hline
sklearn regression & 879 & 85 & 541 & 7 & 0 & 20 & 18 & 0 & 0  \\
\hline
IAI classification  & 19 & 13 & 96 & 4 & 0 & 82 & 3 & 135 & 10  \\
\hline
IAI regression  & 333 & 77 & 367 & 5 & 0 & 25 & 15 & 2 & 0  \\
\hline
Orange  & 10 & 12139 & 405 & 3 & 16 & 1 & 0 & 13 & 0  \\
\hline
Boomer  & 21 & 30 & 180 & 11 & 42 & 6 & 0 & 20 & 0  \\
\hline
\hline
\end{tabular}

\end{table}

\begin{figure}[t]
    \centering
    \begin{minipage}{1\textwidth}
        \centering
        \caption{Box Plot of TR\label{fig:tr}.}
        \includegraphics[width=\linewidth]{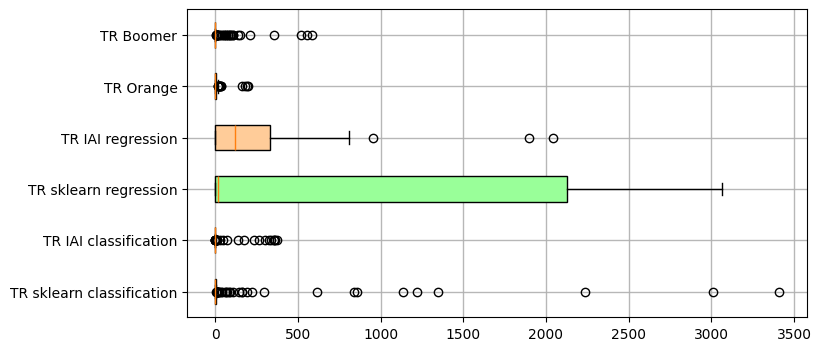}
    \end{minipage}
\end{figure}

\subsection{Rule and Literal Redundancy}
We are interested in this section in the evaluation of the 
presence of redundant rules and locally/globally redundant literals,
their correlations with other characteristics, 
as well as the efficiency of our approach. 

\paragraph{Redundancy.} 
We note first that  rule redundancy does not occur often 
as we can see in column IR in Table~\ref{tab:summary2}
except for Boomer. 
Figure~\ref{fig:lrdd} represents a box plot of the percentage of
 local (respectively global)
 redundancies PL (respectively PG) for all learning models. 
Orange and Boomer barely exhibit literal redundancies 
(see columns IL and IG in Table~\ref{tab:summary2}). 
Regression models did not show any global literal 
redundancy except for 
2 cases with IAI regression trees. 
This is expected because for a literal to be globally redundant, 
there should be at least two rules predicting the exact same 
value, which is rare in regression. 
Classification trees, however, exhibit a noticeable presence of global redundancy 
(`PG IAI classification' and 'PG sklearn classification'). 
Figure~\ref{fig:lrdd} shows a significant presence of local redundancy 
in all tree models. 
We note that for each prediction task (regression, classification), 
IAI trees have fewer 
local/global redundancies than sklearn trees (in terms of the median and the maximum values). 
This suggests that optimal trees tend to reduce redundancy. 
\begin{figure}[t]
    \centering
    \caption{Local/Global Redundancies\label{fig:lrdd}. The X-axis
     represents the values of PL/PG}
    \includegraphics[width=1\textwidth]{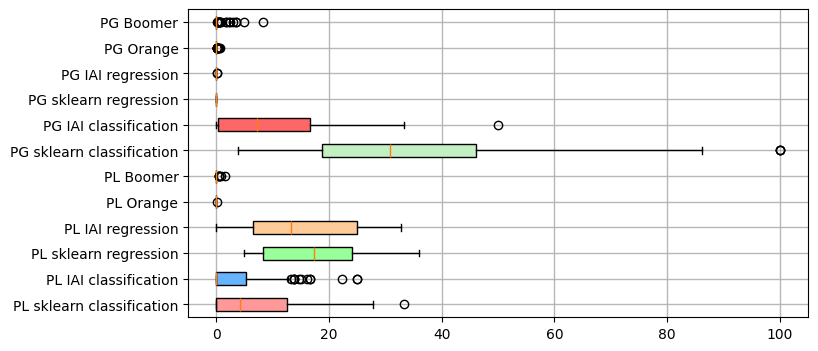}
\end{figure}

\paragraph{Correlations.}
We looked into different correlations between 
local/global redundancy and other statistics. 
We report the results only for models
where at least $30\%$ of 
its decision trees/sets exhibit local/global redundancy. 
There was a moderate negative correlation of global redundancy
with the number of prediction outcomes 
(i.e., size of $\Outcome$)
 with scikit-learn and IAI classification trees. 
Figure~\ref{fig:corr:pl}
shows the most important correlations of 
local redundancy with
the statistics mentioned earlier. 
For instance, on the x-axis, with NR we show the 
correlation 
of the local redundancy 
values found by each model with the number of rules. 
Clearly local redundancy with scikit-learn regression trees 
    highly correlates with NR, NP, BS, RM.  
    IAI regression trees has the same tendency.

\begin{figure}[t]
    \centering
    \caption{Pearson Correlations of Local Redundancies (PL)\label{fig:corr:pl}}
    \includegraphics[width=\textwidth]{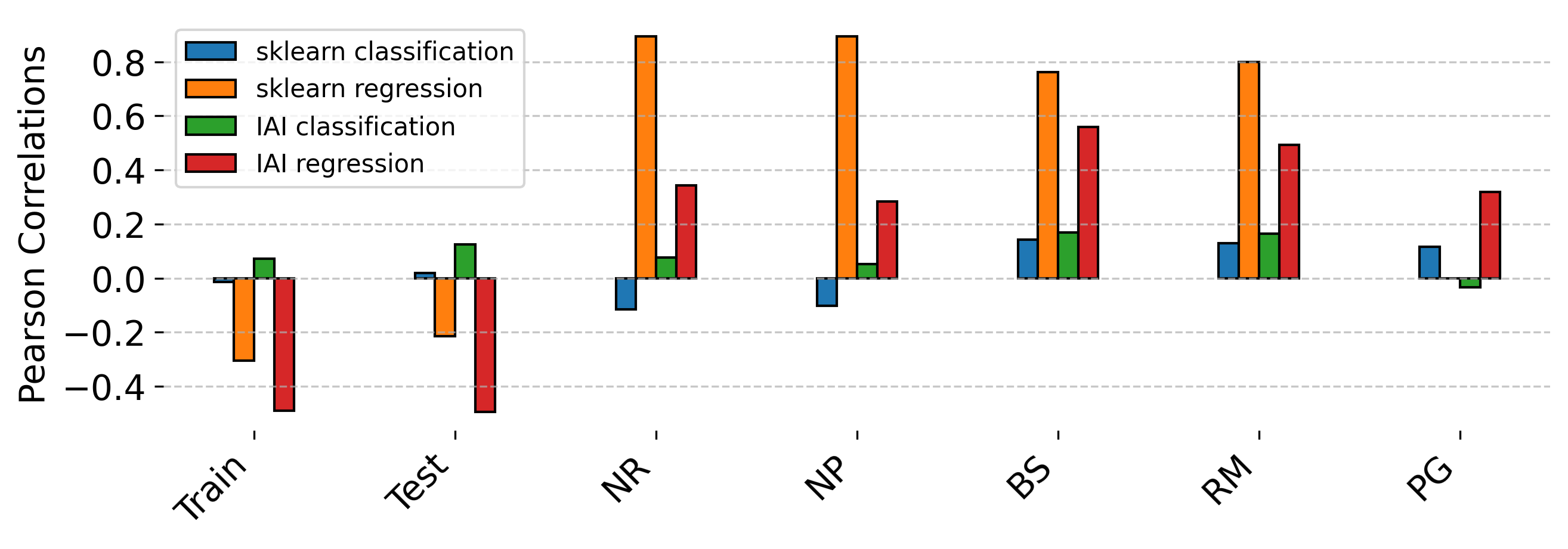}
\end{figure}



\subsection{Negative Overlap}

\mnoteF{Experiments still running. The plots will be updated later.}

\mnoteF{This will be updated later depending on boomer's results. But I put here what needs to be reported.}

We evaluate the presence of negative overlap
on Orange and Boomer 
and their relationship with relevant statistics. 
Boomer timed out on $4$ datasets (emotions, image, scene, yeast) after the one hour 
time limit.
The results are summarized in Table~\ref{tab:summary_overlaps}
for instances that did not timeout. 
The most important observation is the high 
percentage of negative overlap (column PO). 
Indeed, with Boomer decision sets, 
almost every pair of rules with different predictions overlap. 
 Such an observation is worth reporting to 
  the user. 
  The results are less spectacular 
  for orange with an average close to $50\%$
but still worth noting. 
The runtime to find all negative overlap per instance is not negligible. 

\begin{table}[t]
    \centering
    \caption{Summary of the Negative Overlap Results\label{tab:summary_overlaps}}
        \resizebox{\textwidth}{!}
        {
    \begin{tabular}{|c|c|c|c|c|c|c|c|c|c|c|c|c|c|c|}
\hline
  & DS & EX & Train & Test & NR & NP & IR  & TO & TR & BS & RS & RM & PO  \\
\hline
\hline
Orange & 127 & 12 & 70 & 71 & 175 & 2 & 16  & 76 & 10 & 12139 & 405 & 3 & 50  \\
\hline
Boomer & 180 & 16 & 95 & 95 & 97 & 49 & 42  & 0 & 21 & 30 & 180 & 11 & 99  \\
\hline
\hline
\end{tabular}
  
        }
\end{table}

\paragraph{Negative Overlap in Boomer.}
%
As the results in this section confirm (see~\cref{tab:summary_overlaps}),
Boomer~\cite{fuernkranz-ecml20} exhibits extensive negative overlap.
This is to be expected.
In contrast with the approach outlined in this paper, where negative
overlap is targeted as a reason for non-interpretability, Boomer
exploits boosting (and as a result negative overlap) to build
high-accuracy rule ensembles. The theoretical and practical advantages
of boosting are well-known~\cite{freund-ic95,schapire-ml02}, namely
to allow the learning of strong classifiers. As argued in this paper,
a downside of negative overlap (and so of rule ensembles) is that
finding explanations becomes a computationally-hard challenge.
Our experiments are reported for completeness, and confirm the 
previous remarks.

\subsection{Application to Anchor Explanations} 

\mnoteF{I added this short paragraph to describe the relation to Anchor explanations.}

Anchors are well-known model-agnostic explanations 
representing local, “sufficient” conditions for predictions~\cite{18-aaai-Ribeiro}. 
The question we ask here addresses precisely one of the open questions in~\cite{18-aaai-Ribeiro}: 
How to find potentially conflicting anchors?  
To answer this question, we generate anchors for different inputs, 
then apply our approach to find negative overlap between anchors.

We reproduced the exact experiments in~\cite{18-aaai-Ribeiro} with the three datasets: 
\emph{adult} for predicting
whether a person makes > $50K$ annually; 
\emph{rcdv} for predicting recidivism for individuals released
from prison; 
and \emph{lending} for predicting whether a loan on
the Lending Club website will turn out bad. 
For each dataset, four models are used for prediction: 
boosted trees with xgboost, random forest, 
logistic regression,  and neural networks. 
{Each model is built using the exact configuration in the original paper~\cite{18-aaai-Ribeiro}.}
{For each dataset and each model}, we generate all anchors 
of the validation set and look for all negative overlap. 

\begin{table}[t]
    \centering
    \caption{Anchor Experiments\label{fig:anch}}
    \begin{tabular}{|| l | l || c | c | c | c | c | c|  c ||}
\toprule
Learner & Dataset & Train & Test & NR & TO & NO & PO & RM \\ \hline
\midrule
xgboost & recidivism
 & 92.39 & 74.33 & 333 & 0 & 87 & 0.31 & 17 \\
randomforest & recidivism
 & 93.52 & 75.46 & 321 & 0 & 65 & 0.25 & 17 \\
logistic & recidivism
 & 62.59 & 60.00 & 196 & 0 & 735 & 7.81 & 12 \\
nn & recidivism
 & 87.47 & 71.49 & 341 & 1 & 150 & 0.52 & 17 \\ \hline 
xgboost & lending
 & 90.10 & 82.89 & 260 & 0 & 384 & 2.47 & 15 \\
randomforest & lending
 & 91.25 & 83.60 & 278 & 0 & 207 & 1.18 & 15 \\
logistic & lending
 & 82.56 & 83.51 & 50 & 0 & 54 & 9.38 & 14 \\
nn & lending
 & 88.00 & 82.54 & 159 & 0 & 66 & 1.07 & 16 \\\hline 
xgboost & adult
 & 90.35 & 84.26 & 565 & 8 & 3195 & 4.03 & 14 \\
randomforest & adult
 & 93.52 & 85.60 & 558 & 7 & 2534 & 3.27 & 13 \\
logistic & adult
 & 83.00 & 82.98 & 378 & 3 & 2788 & 7.86 & 13 \\
nn & adult
 & 92.47 & 83.62 & 597 & 11 & 3212 & 3.61 & 14 \\\hline 
\bottomrule
\end{tabular}
  
\end{table}

\jnoteF{We need more detail...\\
  What is Train? What is Test? What is \#R? What is \#O? What is \%O?
  Finally, what is RM?
}

\mnoteF{These are statistics that I used across all experiments. They are defined at the very beginning of this section.}

\jnoteF{Ok. Then, highlight that whenever summarizing the results in
  tables/plots.}

Table~\ref{fig:anch} presents the results for each dataset and 
each model. 
As we can see, negative overlap in Anchor explanations is present 
in all use cases. 
 Often, anchors of random forests 
 exhibit the lowest percentage of negative overlap, 
 whereas those of logistic regression have the highest percentage.
We also observe that the best (and respectively, worst) 
models in terms of prediction quality 
tend to have the lowest (respectively, highest)
 percentages of negative overlap.
 These observations suggest that the quality of Anchor 
 explanations depends on the prediction quality of the learner/model.

\section{Conclusions}~\label{sec:conc}

This paper investigates the occurrence of negative facets of decision
sets, namely negative overlap and (global or local) literal
redundancy.
Dedicated algorithms for their identification are proposed.
Furthermore, the paper reveals the tight relationship between decision 
sets for which manual explanations can be devised, and the
non-existence of the aforementioned negative facets.
A first set of experiments confirms that these negative facets occur
ubiquitously in existing implementations of decision sets, thus
rendering unrealistic the manual identification of explanations.
A second set of experiments confirms that the explanations obtained
with the well-known explainer Anchors will also exhibit the same
negative facets.

\paragraph{Acknowledgements.}
Mohamed Siala would like to thank INSA Toulouse for funding his
research visit to the University of Lleida.
This work was supported in part by the
MCIN/AEI/10.13039/501100011033/FEDER, UE under the project
PID2022-139835NB-C22.
This work was supported in part by the Spanish Government under
grant PID 2023-152814OB-I00.
The authors at University of Lleida would like to thank the 
Catalan Government for the quality accreditation given to their 
research group GREiA (2021 SGR 1615).
%
%
%
%

\bibliographystyle{apalike}  
\bibliography{references,xtrarefs}  

\end{document}